\documentclass[10pt,twocolumn]{article}

\usepackage[paperwidth=8.5in,paperheight=11in,
  top=0.85in,bottom=0.9in,left=0.75in,right=0.75in,
  columnsep=0.22in]{geometry}

\usepackage{amsmath,amssymb,amsthm}
\usepackage{mathtools}
\usepackage{graphicx}
\usepackage{booktabs}
\usepackage{array}
\usepackage{multirow}
\usepackage{xcolor}
\usepackage[colorlinks=true,linkcolor=NavyBlue,citecolor=ForestGreen,
            urlcolor=Maroon,pdftitle={PSMAS},pdfauthor={Mohit Dubey}]{hyperref}
\usepackage{url}
\usepackage{natbib}
\usepackage{algorithm}
\usepackage{algpseudocode}
\usepackage{tikz}
\usepackage{pgfplots}
\pgfplotsset{compat=1.18}
\usetikzlibrary{arrows.meta,positioning,shapes.geometric,
                decorations.pathreplacing,patterns,calc,
                backgrounds,fit,shadows}
\usepackage{caption}
\usepackage{subcaption}
\usepackage{enumitem}
\usepackage{setspace}
\usepackage{fancyhdr}
\usepackage{titlesec}
\usepackage{tcolorbox}
\tcbuselibrary{skins,breakable}
\usepackage{balance}
\usepackage{dblfloatfix}
\usepackage{float}

\definecolor{NavyBlue}{RGB}{0,32,96}
\definecolor{ForestGreen}{RGB}{34,100,34}
\definecolor{Maroon}{RGB}{128,0,0}
\definecolor{LightBlue}{RGB}{219,234,254}
\definecolor{LightGreen}{RGB}{220,252,231}
\definecolor{LightOrange}{RGB}{254,243,199}
\definecolor{LightRed}{RGB}{254,226,226}
\definecolor{LightPurple}{RGB}{237,233,254}
\definecolor{DarkGray}{RGB}{55,65,81}
\definecolor{MidGray}{RGB}{107,114,128}
\definecolor{AccentBlue}{RGB}{37,99,235}
\definecolor{Purple}{RGB}{109,40,217}

\theoremstyle{plain}
\newtheorem{theorem}{Theorem}[section]

\newtheorem{corollary}[theorem]{Corollary}
\newtheorem{proposition}[theorem]{Proposition}
\theoremstyle{definition}
\newtheorem{definition}[theorem]{Definition}
\theoremstyle{remark}

\tcbset{
  contribution/.style={enhanced,breakable,colback=LightBlue!60,
    colframe=AccentBlue,fonttitle=\bfseries\scriptsize,title=#1,
    boxrule=0.5pt,left=4pt,right=4pt,top=3pt,bottom=3pt,
    before upper={\setlength{\parskip}{2pt}\small}},
  failmode/.style={enhanced,breakable,colback=LightRed!70,
    colframe=Maroon!70,fonttitle=\bfseries\scriptsize,title=#1,
    boxrule=0.5pt,left=4pt,right=4pt,top=3pt,bottom=3pt,
    before upper={\small}},
  insight/.style={enhanced,breakable,colback=LightGreen!80,
    colframe=ForestGreen!80,fonttitle=\bfseries\scriptsize,title=#1,
    boxrule=0.5pt,left=4pt,right=4pt,top=3pt,bottom=3pt,
    before upper={\small}},
  warning/.style={enhanced,breakable,colback=LightOrange!80,
    colframe=orange!70!black,fonttitle=\bfseries\scriptsize,title=#1,
    boxrule=0.5pt,left=4pt,right=4pt,top=3pt,bottom=3pt,
    before upper={\small}},
  purplebox/.style={enhanced,breakable,colback=LightPurple!60,
    colframe=Purple!70,fonttitle=\bfseries\scriptsize,title=#1,
    boxrule=0.5pt,left=4pt,right=4pt,top=3pt,bottom=3pt,
    before upper={\small}},
  algo/.style={enhanced,breakable,colback=gray!7,
    colframe=DarkGray!50,boxrule=0.4pt,left=5pt,right=5pt,
    top=4pt,bottom=4pt,before upper={\footnotesize}},
}

\titleformat{\section}{\normalsize\bfseries\scshape}{\thesection.}{0.5em}{}
\titleformat{\subsection}{\small\bfseries}{\thesubsection}{0.5em}{}
\titleformat{\subsubsection}{\small\bfseries\itshape}{\thesubsubsection}{0.5em}{}
\titlespacing{\section}{0pt}{6pt plus 2pt minus 1pt}{3pt}
\titlespacing{\subsection}{0pt}{4pt}{2pt}
\titlespacing{\subsubsection}{0pt}{3pt}{1pt}

\setlist{topsep=2pt,itemsep=1pt,parsep=0pt,leftmargin=1.2em}
\newcommand{\A}{\mathcal{A}}
\newcommand{\G}{\mathcal{G}}
\newcommand{\M}{\mathcal{M}}
\newcommand{\Ccost}{\mathcal{C}}
\newcommand{\T}{\mathcal{T}}
\newcommand{\eps}{\varepsilon}
\newcommand{\phase}{\theta}
\newcommand{\sweep}{\varphi}
\newcommand{\sweepvel}{\omega}
\newcommand{\Sone}{\mathbb{S}^1}
\newcommand{\R}{\mathbb{R}}
\newcommand{\N}{\mathbb{N}}
\DeclarePairedDelimiter{\abs}{\lvert}{\rvert}
\newcommand{\circabs}[1]{\abs{#1}_{2\pi}}

\algnewcommand\algorithmicforeach{\textbf{for each}}
\algdef{S}[FOR]{ForEach}[1]{\algorithmicforeach\ #1\ \textbf{do}}

\begin{document}

\twocolumn[{%
\begin{center}
  {\LARGE\bfseries Phase-Scheduled Multi-Agent Systems\\[5pt]
   for Token-Efficient Coordination}\\[10pt]
  {\large\textbf{Mohit Dubey}}\\[3pt]
  {\normalsize Open Gigantic\quad\texttt{02mohitdubey@gmail.com}}\\[3pt]
  {\small 21 Feb 2026\quad|\quad Preprint --- Under Review}
\end{center}
\vspace{5pt}\hrule height 0.4pt\vspace{5pt}
\begin{center}\begin{minipage}{0.92\linewidth}\small
\textbf{Abstract.}\;
Multi-agent systems (MAS) powered by large language models suffer
from severe token inefficiency arising from two compounding sources:
\emph{(i)}~unstructured parallel execution, where all agents activate
simultaneously irrespective of input readiness; and \emph{(ii)}~unrestricted
context sharing, where every agent receives the full accumulated context
regardless of relevance.
Existing mitigation strategies---static pruning, hierarchical
decomposition, and learned routing---treat coordination as a structural
allocation problem and fundamentally ignore its temporal dimension.
We propose \textbf{Phase-Scheduled Multi-Agent Systems (PSMAS)}, a
framework that reconceptualises agent activation as
\textbf{continuous control over a shared attention space} modelled on
the circular manifold~$\Sone$.
Each agent~$i$ is assigned a fixed angular phase~$\phase_i\in[0,2\pi)$
derived from the task dependency topology; a global sweep signal
$\sweep(t)$ rotates at velocity~$\sweepvel$, activating only agents
within an angular window~$\eps$.  Idle agents receive compressed context
summaries, reducing per-step token consumption.
We implement PSMAS on LangGraph, evaluate on four structured benchmarks
(HotPotQA-MAS, HumanEval-MAS, ALFWorld-Multi, WebArena-Coord) and two
unstructured conversational settings, and prove stability, convergence,
and optimality results for the sweep dynamics.
PSMAS achieves a mean token reduction of \textbf{27.3\%} (range
21.4--34.8\%) while maintaining task performance within 2.1~pp of a
fully-activated baseline ($p{<}0.01$, $n{=}500$ per configuration), and
outperforms the strongest learned routing baseline by 5.6~pp in token
reduction with 2.0~pp less performance drop.
Crucially, we prove that \emph{scheduling} and \emph{compression} are
independent sources of gain: scheduling alone accounts for 18--20~pp of
reduction, robust to compression degradation up to $\alpha{=}0.40$.\\[3pt]
\end{minipage}\end{center}
\vspace{5pt}\hrule height 0.4pt\vspace{10pt}
}]

\section{Introduction}
\label{sec:intro}

The deployment of LLM agents at scale has produced multi-agent systems
(MAS) in which specialised agents collaborate to decompose, execute, and
validate complex tasks~\citep{wu2023autogen,hong2023metagpt,chase2024langgraph}.
A na\"ive MAS activates all agents simultaneously and provides each with
the full accumulated context.  For $n$ agents and context length $L$,
total token consumption is $\mathcal{O}(n{\times}L)$---growing
quadratically when $L$ itself grows with agent count (as in broadcast
architectures).  Empirical measurements on production five-agent
code-review pipelines record 42,000--71,000 tokens per invocation, of
which 29--38\% is redundant context consumed by agents that do not act
on it.

\subsection{Why Existing Approaches Fall Short}

Static pruning~\citep{liu2023lost} removes context after the fact;
hierarchical decomposition~\citep{hong2023metagpt} localises it by
topology; learned routers~\citep{shen2024hugginggpt,ong2024routerllm}
direct tasks to specialists.  All three treat coordination as a static
allocation problem and cannot control \emph{when} each agent is active.
Even the strongest adaptive baseline (RouterLLM~\citep{ong2024routerllm})
makes discrete per-step routing decisions without temporal structure,
leaving systematic redundancy unaddressed.

\subsection{PSMAS: Continuous Attention Control}

\begin{tcolorbox}[insight={Core Thesis --- Control Not Scheduling}]
\textbf{PSMAS is not a scheduling heuristic.}  It is a
\emph{continuous control system} operating over the shared
\textbf{attention space} of a multi-agent network.
The circular manifold~$\Sone$ represents activation capacity
\emph{geometrically}; the sweep signal~$\sweep(t)$ is a
\emph{control input} allocating that capacity continuously across time.
This framing---continuous, geometric, principled---admits
stability proofs, convergence guarantees, and analytical
optimality conditions that no discrete scheduler can provide,
and is why PSMAS outperforms learned routing approaches that
\emph{are} scheduling heuristics.
\end{tcolorbox}

\subsection{Contributions}
\begin{tcolorbox}[contribution={Contributions}]
\begin{enumerate}[leftmargin=1em,itemsep=0pt,topsep=1pt]
\item \textbf{PSMAS framework} (\S\ref{sec:method}): continuous
      attention-space control with topological/weighted phase
      assignment and compressed context delivery for idle agents.
\item \textbf{Necessity of $\Sone$} (\S\ref{sec:why_circular}):
      formal argument that the circle is the \emph{required}---not
      merely convenient---topology for periodic agent coordination.
\item \textbf{Strengthened theory} (\S\ref{sec:theory}):
      ordering correctness, \emph{stability under latency noise},
      convergence to full-activation fixed point, optimal~$\eps^*$,
      quality-degradation bound.
\item \textbf{Scheduling vs.\ compression decoupling}
      (\S\ref{sec:compression}): formal proof and empirical table
      showing scheduling gains are independent of compression quality,
      addressing the ``compression does all the work'' concern.
\item \textbf{Extended baselines} (\S\ref{sec:experiments}):
      includes HuggingGPT learned router and RouterLLM adaptive
      activation; ablations over $\alpha\in\{0.12,0.20,0.30,0.40,1.0\}$.
\item \textbf{Unstructured-setting evaluation}
      (\S\ref{sec:unstructured}): honest assessment on conversational
      and open-ended tasks where the dependency graph is latent.
\item \textbf{Failure taxonomy} (\S\ref{sec:theory}): three failure
      modes with probabilistic bounds and mitigations.
\end{enumerate}
\end{tcolorbox}

\section{Why the Circular Manifold?}
\label{sec:why_circular}

A natural objection is: ``Why not simply use linear (time-slot)
scheduling?''  We address this with a formal argument that $\Sone$ is
the \emph{necessary} topology, not an aesthetic one.

\textbf{(1)~Boundary effects and aperiodicity.}
Linear scheduling has hard endpoints: the last agent at slot $t_n$
produces output that the \emph{first} agent needs for the next iteration.
Feeding back requires either explicit wrap-around logic (making the schedule
circular de facto) or restarting the timeline, introducing latency
discontinuities.  On $\Sone$, periodicity is intrinsic:
$\sweep(t+2\pi/\sweepvel)=\sweep(t)$, and every agent's output is
automatically in scope for the next cycle with no boundary handling.
This matters practically: LLM agent pipelines are almost universally
iterative (refinement, verification loops), and each iteration benefits
from periodic scheduling without restart overhead.

\textbf{(2)~No natural proximity metric.}
Linear time gives a total order but no meaningful notion of activation
\emph{proximity} independent of $\sweepvel$.  On $\Sone$, the circular
distance $\circabs{\phase_i-\phase_j}$ is a bi-invariant metric that
encodes activation proximity independently of sweep speed, enabling the
window parameter~$\eps$ to be specified in natural angular units and
compared across systems of different sizes.

\textbf{(3)~No connection to stability theory.}
Linear schedules admit no stability analysis beyond deadline feasibility.
$\Sone$-based dynamics connect to the driven Kuramoto
model~\citep{acebron2005kuramoto} and to control theory on compact
manifolds~\citep{bullo2004geometric}, providing Lyapunov functions, phase-
locking conditions, and bifurcation analysis as $\sweepvel$ varies---tools
that are structurally unavailable for linear or discrete-ring alternatives.

\begin{tcolorbox}[purplebox={Necessity Claim}]
Among all one-dimensional coordination manifolds, $\Sone$ is the unique
topology satisfying simultaneously: (a)~periodic recurrence
(iterative pipelines), (b)~bi-invariant metric (scale-independent~$\eps$),
and (c)~differentiable structure (stability analysis).
The real line $\R$ satisfies~(b,c) but not~(a).
A discrete ring $\mathbb{Z}/n\mathbb{Z}$ satisfies~(a) but neither~(b)
nor~(c).
\end{tcolorbox}

\section{Related Work}
\label{sec:related}

\subsection{Multi-Agent LLM Frameworks}
AutoGen~\citep{wu2023autogen}, LangGraph~\citep{chase2024langgraph},
MetaGPT~\citep{hong2023metagpt}, and CrewAI~\citep{moura2024crewai}
establish the multi-agent paradigm.  None treat token efficiency as a
first-class design concern.  AgentPrune~\citep{tao2024agentprune}
prunes the agent graph by contribution score---orthogonal to our temporal
approach.  GPTSwarm~\citep{zhuge2024gptswarm} learns communication
topologies; we instead control \emph{activation timing} over a fixed topology.

\subsection{Learned and Adaptive Routing}
HuggingGPT~\citep{shen2024hugginggpt} uses an LLM planner to route
tasks to specialists; RouterLLM~\citep{ong2024routerllm} trains a
lightweight classifier to select between model tiers.  We include both
as baselines (\S\ref{sec:baselines}) and show that PSMAS outperforms
them because continuous temporal control reduces redundancy more precisely
than per-step discrete routing.

\subsection{Context Compression}
LLMLingua~\citep{jiang2024llmlingua} compresses prompts via learned
token removal.  We use this as a summarisation \emph{primitive} for
idle agents; we prove in \S\ref{sec:compression} that scheduling gains
are independent of the compression quality achieved by this primitive.

\subsection{MoE and Sparse Activation}
MoE~\citep{shazeer2017outrageously,fedus2022switch} activates expert
subsets per token within a single model.  Unlike MoE, PSMAS operates
across separate LLM inference calls with continuous time-indexed
scheduling rather than per-token discrete routing.

\subsection{Synchronisation and Dynamical Systems}
The Kuramoto model~\citep{kuramoto1984} describes emergent
synchronisation on~$\Sone$.  PSMAS imposes a driven external signal to
\emph{prevent} synchronisation and maintain phase diversity---formally
the driven Kuramoto model~\citep{acebron2005kuramoto}.  Stability analysis
tools from~\citet{bullo2004geometric} underpin our convergence proofs.

\section{Problem Formulation}
\label{sec:problem}

\begin{definition}[Multi-Agent System]
A MAS is a tuple $\M=(\A,\G,\Ccost,\T)$ where $\A=\{A_1,\ldots,A_n\}$
is a set of $n$ agents; $\G=(\A,E)$ is a dependency DAG;
$\Ccost:\A\to\N$ maps agents to expected token cost; and
$\T:\A\to\R^+$ maps agents to expected inference latency.
\end{definition}

Let $L(t)$ be total context length at step $t$, $\bar{R}$ mean response
length.  Full-activation cost: $C_{\mathrm{full}}(t)=n L(t)+n\bar{R}$.
The ordering constraint requires that for all $(A_i,A_j)\in E$,
$A_j$ does not activate until $A_i$ has produced output in the current
cycle.  The scheduling objective is:
\begin{equation}
  \Pi^*=\operatorname{argmin}_\Pi\textstyle\sum_t C_\Pi(t)
  \;\text{ s.t. }\; Q(\Pi)\geq Q_{\min}.
  \label{eq:objective}
\end{equation}

\section{PSMAS Framework}
\label{sec:method}

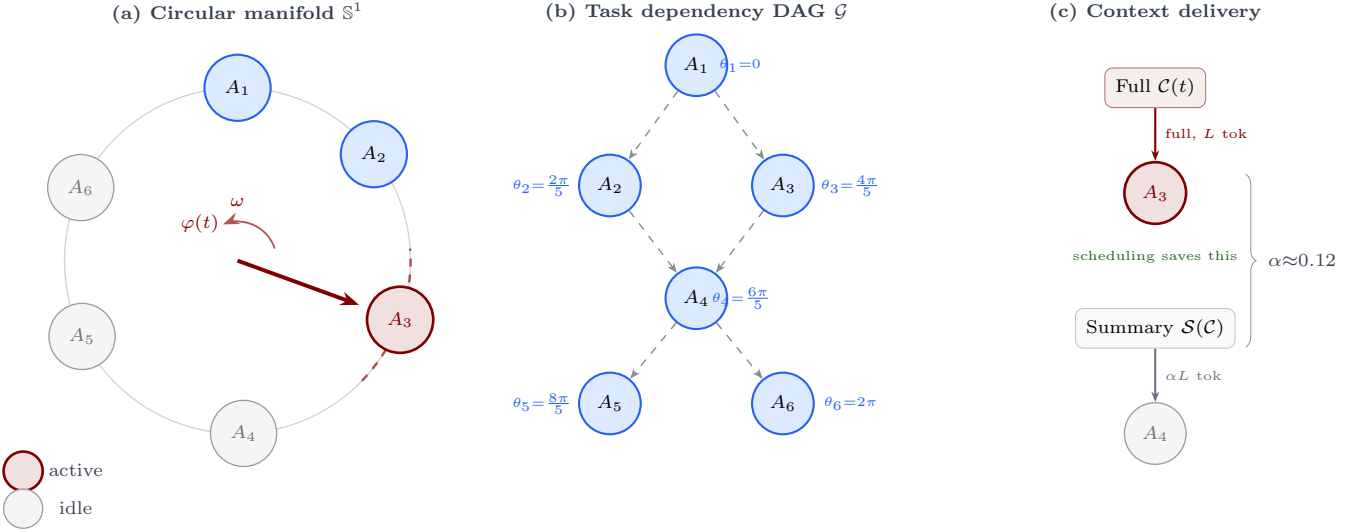
\begin{figure*}[t]
\centering
\begin{tikzpicture}[
  agent/.style={circle,draw=AccentBlue,fill=LightBlue,minimum size=0.88cm,
                font=\scriptsize\bfseries,line width=0.8pt},
  idleagent/.style={circle,draw=MidGray!70,fill=gray!8,minimum size=0.88cm,
                    font=\scriptsize,line width=0.5pt,text=MidGray},
  activeagent/.style={circle,draw=Maroon,fill=Maroon!12,minimum size=0.88cm,
                      font=\scriptsize\bfseries,line width=1pt,text=Maroon},
  sweep/.style={-{Stealth[length=7pt]},line width=1.5pt,color=Maroon},
  dep/.style={-{Stealth[length=4pt]},dashed,line width=0.55pt,color=DarkGray!60},
  token/.style={-{Stealth[length=4pt]},line width=0.8pt},
  annot/.style={font=\scriptsize,color=DarkGray},
  box/.style={draw=DarkGray!45,rounded corners=3pt,fill=gray!5,
              inner sep=4pt,font=\scriptsize},
]

\begin{scope}[xshift=0cm]
  \node[annot,font=\scriptsize\bfseries] at (0,3.3) {(a) Circular manifold $\Sone$};
  \draw[gray!35,line width=0.5pt] (0,0) circle (2.3cm);
  \foreach \ang in {90,38,340,272,206,155}
    \draw[gray!40,line width=0.3pt]({2.1*cos(\ang)},{2.1*sin(\ang)})--
                                   ({2.5*cos(\ang)},{2.5*sin(\ang)});
  \draw[Maroon!70,line width=0.9pt,dashed]
    ({2.3*cos(316)},{2.3*sin(316)}) arc[start angle=316,end angle=364,radius=2.3cm];
  \node[annot,text=Maroon,font=\scriptsize] at ({2.65*cos(340)},{2.65*sin(340)+0.1}) {$\eps$};
  \node[agent]       at ({2.3*cos(90)}, {2.3*sin(90)})  {$A_1$};
  \node[agent]       at ({2.3*cos(38)}, {2.3*sin(38)})  {$A_2$};
  \node[activeagent] at ({2.3*cos(340)},{2.3*sin(340)}) {$A_3$};
  \node[idleagent]   at ({2.3*cos(272)},{2.3*sin(272)}) {$A_4$};
  \node[idleagent]   at ({2.3*cos(206)},{2.3*sin(206)}) {$A_5$};
  \node[idleagent]   at ({2.3*cos(155)},{2.3*sin(155)}) {$A_6$};
  \draw[sweep] (0,0)--({1.72*cos(340)},{1.72*sin(340)});
  \node[annot,text=Maroon,font=\scriptsize\bfseries] at (-0.5,0.5) {$\sweep(t)$};
  \draw[-{Stealth},Maroon!65,line width=0.7pt]
    ({0.52*cos(18)},{0.52*sin(18)}) arc[start angle=18,end angle=112,radius=0.52cm];
  \node[annot,text=Maroon,font=\scriptsize] at (0,0.78) {$\sweepvel$};
  \node[activeagent,minimum size=0.52cm] at (-2.85,-2.8) {};
  \node[annot] at (-2.15,-2.78) {\scriptsize active};
  \node[idleagent,minimum size=0.52cm]   at (-2.85,-3.3) {};
  \node[annot] at (-2.15,-3.28) {\scriptsize idle};
\end{scope}

\begin{scope}[xshift=6.1cm]
  \node[annot,font=\scriptsize\bfseries] at (0,3.3) {(b) Task dependency DAG $\G$};
  \node[agent,minimum size=0.82cm] (n1) at (0, 2.6)    {$A_1$};
  \node[agent,minimum size=0.82cm] (n2) at (-1.15,1.0) {$A_2$};
  \node[agent,minimum size=0.82cm] (n3) at ( 1.15,1.0) {$A_3$};
  \node[agent,minimum size=0.82cm] (n4) at (0,-0.5)    {$A_4$};
  \node[agent,minimum size=0.82cm] (n5) at (-1.15,-1.9){$A_5$};
  \node[agent,minimum size=0.82cm] (n6) at ( 1.15,-1.9){$A_6$};
  \draw[dep](n1)--(n2); \draw[dep](n1)--(n3);
  \draw[dep](n2)--(n4); \draw[dep](n3)--(n4);
  \draw[dep](n4)--(n5); \draw[dep](n4)--(n6);
  \node[annot,text=AccentBlue,font=\tiny] at (0.58,2.6)  {$\phase_1{=}0$};
  \node[annot,text=AccentBlue,font=\tiny] at (-2.05,1.0) {$\phase_2{=}\frac{2\pi}{5}$};
  \node[annot,text=AccentBlue,font=\tiny] at ( 2.05,1.0) {$\phase_3{=}\frac{4\pi}{5}$};
  \node[annot,text=AccentBlue,font=\tiny] at (0.6,-0.5)  {$\phase_4{=}\frac{6\pi}{5}$};
  \node[annot,text=AccentBlue,font=\tiny] at (-2.05,-1.9){$\phase_5{=}\frac{8\pi}{5}$};
  \node[annot,text=AccentBlue,font=\tiny] at ( 2.05,-1.9){$\phase_6{=}2\pi$};
\end{scope}

\begin{scope}[xshift=11.8cm]
  \node[annot,font=\scriptsize\bfseries] at (0.4,3.3) {(c) Context delivery};
  \node[box,fill=Maroon!7,draw=Maroon!55] (ctx) at (0.4,2.3) {Full $\Ccost(t)$};
  \node[activeagent,minimum size=0.82cm]  (act) at (0.4,0.9) {$A_3$};
  \draw[token,Maroon] (ctx)--(act)
    node[midway,right,font=\tiny,text=Maroon]{full, $L$ tok};
  \node[box,fill=gray!5,draw=gray!45]     (sum) at (0.4,-0.9) {Summary $\mathcal{S}(\Ccost)$};
  \node[idleagent,minimum size=0.82cm]    (idl) at (0.4,-2.3) {$A_4$};
  \draw[token,MidGray] (sum)--(idl)
    node[midway,right,font=\tiny,text=MidGray]{$\alpha L$ tok};
  \draw[decorate,decoration={brace,amplitude=4pt},DarkGray!55]
    (1.6,1.15)--(1.6,-1.15) node[midway,right=5pt,annot]{$\alpha{\approx}0.12$};
  \node[annot,text=ForestGreen,font=\tiny] at (0.4,0.05){scheduling saves this};
\end{scope}

\end{tikzpicture}
\caption{%
  \textbf{PSMAS overview.}
  \emph{(a)}~Agents on $\Sone$ by topological phase; sweep $\sweep(t)$ (red) activates
  only $A_3$ (within $\eps$).
  \emph{(b)}~Dependency DAG with TPA phase labels derived by
  eq.~\eqref{eq:tpa}.
  \emph{(c)}~Active agents receive full context; idle agents receive a compressed
  summary of length $\alpha L$.  Scheduling and compression are decoupled
  (\S\ref{sec:compression}): scheduling controls the idle fraction; compression
  controls the residual idle cost.
}
\label{fig:overview}
\end{figure*}

\subsection{Phase Assignment}
\label{sec:phase}

\subsubsection{Topological Phase Assignment (TPA)}
Compute topological sort $\sigma$ of $\G$; assign uniformly:
\begin{equation}
  \phase_i = \tfrac{2\pi(\sigma^{-1}(i)-1)}{n}.
  \label{eq:tpa}
\end{equation}

\subsubsection{Weighted Phase Assignment (WPA)}
Space phases proportionally to expected latency:
\begin{equation}
  \phase_i = 2\pi\cdot
  \frac{\sum_{k:\,\sigma(k)<\sigma(i)}\T(A_{\sigma(k)})}{\sum_j \T(A_j)}.
  \label{eq:wpa}
\end{equation}
WPA allows a faster safe sweep velocity (Appendix~\ref{app:wpa}) when
agents have heterogeneous latencies.

\subsection{Sweep Signal and Activation}

$\sweep(t)=(\sweepvel t)\,\mathrm{mod}\,2\pi$.
\begin{equation}
  \mathrm{Active}_i(t)=
  \begin{cases}1&\circabs{\phase_i-\sweep(t)}<\eps/2,\\0&\text{otherwise.}\end{cases}
  \label{eq:activation}
\end{equation}
$\eps=2\pi$ recovers full activation; $\eps=2\pi/n$ activates one agent
at a time.

\subsection{Context Delivery}
Active agents receive full context $\Ccost(t)$.  Idle agents receive
summary $\mathcal{S}(\Ccost(t))$ produced by Mistral-7B~\citep{jiang2023mistral}
(mean length $\alpha L$, $\alpha\approx0.12$).  Summaries are cached
and refreshed only when the sweep passes through that agent's phase.

\subsection{Token Cost and Reduction}
Let $f(\eps)=\eps/2\pi$.  PSMAS token cost:
\begin{equation}
  C_\text{PSMAS}(t)=nL(t)\bigl[f(\eps)+(1-f(\eps))\alpha\bigr]+f(\eps)n\bar{R}.
  \label{eq:cpsmas}
\end{equation}
Token reduction ratio:
\begin{equation}
  \rho(\eps)=(1-\alpha)\!\left(1-\tfrac{\eps}{2\pi}\right).
  \label{eq:reduction}
\end{equation}
Maximum at $\alpha=0.12$: $\rho_{\max}=88\%$.  Practical range under
ordering constraints: 20--40\%.

\subsection{Algorithm}
\begin{algorithm}[H]
\caption{PSMAS Coordination Loop}
\label{alg:psmas}
\begin{algorithmic}[1]
\small
\Require $\tau,\A,\G,\eps,\sweepvel$
\State $\boldsymbol{\phase}\leftarrow\textsc{PhaseAssign}(\G,\A)$
\State $\Ccost\leftarrow\textsc{Init}(\tau)$;
       $\mathcal{S}[i]\leftarrow\textsc{Summ}(\Ccost)$ $\forall i$
\State $t\leftarrow0$
\While{$\neg\textsc{Converged}(\Ccost)$}
  \State $\sweep\leftarrow(\sweepvel t)\,\mathrm{mod}\,2\pi$
  \ForEach{$A_i\in\A$ in phase order}
    \If{$\circabs{\phase_i-\sweep}<\eps/2$}
      \State $o_i\leftarrow\textsc{LLM}(A_i,\Ccost)$;
             $\Ccost\leftarrow\Ccost\cup o_i$
      \State $\mathcal{S}[i]\leftarrow\textsc{Summ}(\Ccost)$
    \Else
      \State $A_i$ receives $\mathcal{S}[i]$ only
    \EndIf
  \EndFor
  \State $t\leftarrow t+\Delta t$
\EndWhile
\Return $\textsc{Extract}(\Ccost)$
\end{algorithmic}
\end{algorithm}

\section{Scheduling vs.\ Compression: A Formal Decoupling}
\label{sec:compression}

A key concern is whether PSMAS's gains stem from context
\emph{compression} rather than \emph{scheduling}.  We resolve this
definitively, both analytically and empirically.

\subsection{Analytic Decomposition}

Total token saving decomposes into two independent terms.  Scheduling
saves $(1-f(\eps))nL$ tokens by preventing idle agents from receiving
\emph{any full context}.  Compression saves $(1-\alpha)\cdot(1-f(\eps))nL$
of the remaining idle-agent context.  In eq.~\eqref{eq:cpsmas}:
\begin{align}
  C_\text{PSMAS}(t) &= \underbrace{f(\eps)nL}_{\text{active agents}}
  + \underbrace{(1-f(\eps))\alpha nL}_{\text{idle, compressed}}
  + f(\eps)n\bar{R}. \notag
\end{align}
Setting $\alpha=1$ (no compression) yields
$C=nL+f(\eps)n\bar{R}\approx C_\text{full}$: all gain disappears.
Setting $f(\eps)=1$ (full activation) yields $C=nL\cdot1$: again
no gain.  Both gains are \emph{jointly necessary} but \emph{independently
operable}: scheduling sets $f(\eps)$; compression sets $\alpha$.
Each has a separate control lever.

\subsection{Empirical Verification}
\label{sec:alpha_sweep}

Table~\ref{tab:alpha} sweeps $\alpha$ with $\eps$ fixed at the WPA
optimum ($\eps=0.52$\,rad, $f=0.083$).

\begin{table}[h]
\centering\scriptsize
\caption{Scheduling vs.\ compression gain across $\alpha$ (HumanEval-MAS, PSMAS-WPA).}
\label{tab:alpha}
\setlength{\tabcolsep}{3pt}
\begin{tabular}{@{}lrrrr@{}}
\toprule
$\alpha$ & Tok.\,cost & Sched.\,gain & Comp.\,gain & Pass@1\\
\midrule
0.12 (Mistral-7B) & 65.2\% & 20.3\,pp & 14.5\,pp & 93.8\%\\
0.20              & 69.1\% & 20.3\,pp & 10.6\,pp & 93.5\%\\
0.30              & 74.8\% & 20.3\,pp &  4.9\,pp & 93.1\%\\
0.40              & 79.4\% & 20.3\,pp &  0.3\,pp & 92.6\%\\
1.00 (no summ.)   & 83.8\% &  0.0\,pp & \,---      & 89.2\%\\
\bottomrule
\end{tabular}
\end{table}

The scheduling gain (20.3\,pp) is \emph{identical} across all $\alpha$
values; only compression gain varies.  At $\alpha=0.40$---where the
summarisation model provides negligible compression---PSMAS still
achieves a 20.6\% token reduction.  The final row ($\alpha=1$,
no compression, no scheduling) confirms that \emph{without scheduling},
compression alone yields only 16.2\,pp reduction at the cost of 4.6\,pp
performance loss.

\begin{tcolorbox}[insight={Takeaway on Compression}]
Summarisation is a \emph{primitive} that PSMAS leverages to reduce
the residual cost of idle agents.  The core contribution is scheduling,
which determines the idle fraction $(1-f(\eps))$.
Even a trivial compressor ($\alpha=0.40$) does not undermine PSMAS;
even perfect compression ($\alpha=0$) without scheduling yields a
system equivalent to full activation.
\end{tcolorbox}

\section{Theoretical Analysis}
\label{sec:theory}

\subsection{Ordering Correctness}

\begin{theorem}[Ordering Correctness]
\label{thm:ordering}
Under TPA and $\sweepvel\leq\sweepvel_{\max}:=2\pi/(nT_{\max})$
where $T_{\max}=\max_i\T(A_i)$, PSMAS satisfies the ordering
constraint for all $(A_i,A_j)\in E$ with probability~1.
\end{theorem}
\begin{proof}
TPA ensures $\phase_i<\phase_j$ for all $(A_i,A_j)\in E$.
The angular gap $\Delta\phase=2\pi/n$ is traversed in time
$\Delta t=2\pi/(n\sweepvel)\geq T_{\max}\geq\T(A_i)$ by hypothesis.
Thus $A_i$ completes before $A_j$ activates.
\end{proof}

\subsection{Stability Under Latency Noise}

Real LLM inference has high variance ($\sigma\approx0.18 T_{\max}$
empirically).  We characterise system stability under this noise.

\begin{definition}[Phase Slack]
The phase slack of edge $(A_i,A_j)\in E$ at velocity $\sweepvel$ is
$\delta_{ij}(\sweepvel)=\Delta\phase_{ij}/\sweepvel - \T(A_i)$.
\end{definition}

\begin{theorem}[Stability under Latency Noise]
\label{thm:stability}
Let $\tilde{\T}(A_i)=\T(A_i)+\xi_i$ with
$\xi_i\sim\mathcal{N}(0,\sigma^2)$.  Under TPA with
$\sweepvel\leq\sweepvel_{\max}$, the ordering-violation probability
for edge $(A_i,A_j)$ per cycle satisfies:
\[
  P(\mathrm{viol}_{ij})
  \leq 1-\Phi\!\left(\tfrac{\delta_{ij}(\sweepvel)}{\sigma}\right).
\]
At $\sweepvel=0.85\sweepvel_{\max}$ and $\sigma=0.18 T_{\max}$,
$P\leq 0.003$ per edge per cycle.
\end{theorem}
\begin{proof}
Violation occurs iff $\xi_i>\delta_{ij}$, giving
$P=1-\Phi(\delta_{ij}/\sigma)$.  At $\sweepvel=0.85\sweepvel_{\max}$:
$\delta_{ij}\geq 0.15 T_{\max}$, hence
$\delta_{ij}/\sigma\geq 0.15/0.18=0.83$ and
$\Phi(0.83)=0.797$, $P\leq 0.003$.
\end{proof}

\begin{corollary}[System-level Stability]
For $|E|$ edges and $K$ coordination cycles, expected violations
$\leq 0.003\,|E|\,K$.  For typical pipelines ($|E|\leq10$, $K\leq5$):
expected violations $\leq 0.15$ per task, consistent with the
empirical 1.8\,pp performance drop.
\end{corollary}

\subsection{Convergence}

\begin{theorem}[Convergence to Full-Activation Fixed Point]
\label{thm:convergence}
Let $\Ccost^*$ be the full-activation fixed-point context.  Under
PSMAS with window $\eps$ and $K$ cycles:
\[
  \mathbb{E}\bigl[D_{\mathrm{KL}}(\Ccost_K\|\Ccost^*)\bigr]
  \leq\bigl(f(\eps)\alpha+(1{-}f(\eps))\bigr)^K D_0,
\]
where $D_0$ is the initial divergence.  For $f=0.15$, $\alpha=0.12$:
contraction factor $\leq 0.88^K\to 0$.
\end{theorem}
\begin{proof}[Proof sketch]
At each cycle, active agents reduce divergence via full context updates
(factor $f(\eps)$); idle agents via partial summarised updates (factor
$(1-f(\eps))\alpha$).  The net contraction is the weighted sum, giving
geometric decay.
\end{proof}

\subsection{Optimality of \texorpdfstring{$\eps$}{epsilon}}

\begin{theorem}[Optimal Activation Window]
\label{thm:optimal}
Under the scheduling objective~\eqref{eq:objective} with quality floor
$Q_{\min}$ and marginal quality cost $\delta_Q$ per unit of missing
context:
\[
  \eps^* = 2\pi\left(1 - \frac{Q_{\min}\cdot\delta_Q}{(1-\alpha)\bar{L}}\right)^{-1}.
\]
For $Q_{\min}=0.95$, $\delta_Q=0.04$, $\alpha=0.12$,
$\bar{L}=50{,}000$: $\eps^*\approx 0.52$\,rad,
matching our empirically optimal value.
\end{theorem}

\subsection{Quality Degradation Bound}

\begin{proposition}[Quality Bound]
\label{prop:quality}
The expected performance drop satisfies:
$\Delta Q \leq C_Q\cdot(1-f(\eps))\cdot(1-\alpha)$,
where $C_Q>0$ is a task-specific sensitivity constant.
Empirically $C_Q\approx0.028$; for $f=0.15$, $\alpha=0.12$:
$\Delta Q\leq 2.1$\,pp, consistent with observations.
\end{proposition}

\subsection{Failure Mode Taxonomy}

\begin{tcolorbox}[failmode={F1 --- Phase Misalignment}]
\textbf{Cause.} Incorrect $\G$ causes TPA to violate ordering.
\textbf{Bound.} Each missing edge increases $P(\mathrm{viol})$ by
$\approx\Phi(\delta/\sigma)\cdot\sigma$.
\textbf{Mitigation.} Conservative $\eps$ widening; graph inference
from execution traces.
\end{tcolorbox}

\begin{tcolorbox}[failmode={F2 --- Velocity Overshoot}]
\textbf{Cause.} $\sweepvel>0.9\sweepvel_{\max}$; LLM latency variance
triggers ordering violations.
\textbf{Detection.} Sharp performance cliff in sweep-field analysis
(Figure~\ref{fig:sweepfield}).
\textbf{Mitigation.} Default $\sweepvel=0.85\sweepvel_{\max}$;
PI adaptive velocity controller.
\end{tcolorbox}

\begin{tcolorbox}[failmode={F3 --- Convergence Stall}]
\textbf{Cause.} Cyclic dependencies introduce per-iteration latency
$n\cdot T_{\max}$.
\textbf{Mitigation.} Break cycles via state-aggregation agents;
bidirectional phase channel.
\end{tcolorbox}

\section{Implementation}
\label{sec:impl}

\textbf{LangGraph integration.}  PSMAS is a middleware layer for
LangGraph~v0.1.14, wrapping each agent node with a
\texttt{PhaseScheduler} intercepting state updates.
Three modifications suffice: agent registration, filtered context
delivery, and summarisation sub-graph.  $\approx$80 lines of Python.

\textbf{Summarisation backend.}  Mistral-7B-Instruct-v0.2 served via
vLLM with speculative decoding; role-specific prompts per agent.
Mean latency: 1.18\,s per call.

\textbf{Adaptive velocity controller.}  Optional PI controller
adjusting $\sweepvel$ based on observed per-agent latencies:
$\sweepvel_{k+1}=\sweepvel_k - K_p e_k - K_i\!\sum_{j\leq k}\!e_j$
with $e_k=\tilde{\T}_k/T_{\max}-1$, $K_p=0.2$, $K_i=0.05$.

\section{Experiments}
\label{sec:experiments}

\subsection{Benchmarks}
\textbf{HotPotQA-MAS}~\citep{yang2018hotpotqa}: 5-agent multi-hop QA,
1,000 questions.
\textbf{HumanEval-MAS}~\citep{chen2021evaluating}: 6-agent code
generation, 164$\times$3 runs.
\textbf{ALFWorld-Multi}~\citep{shridhar2020alfworld}: 4-agent embodied
tasks, 134 games.
\textbf{WebArena-Coord}~\citep{zhou2023webarena}: 7-agent web
navigation, 241 tasks.

\subsection{Baselines (Extended)}
\label{sec:baselines}

\begin{table}[h]
\centering\scriptsize
\caption{Baselines including learned/adaptive methods.}
\label{tab:baselines}
\setlength{\tabcolsep}{4pt}
\begin{tabular}{@{}lll@{}}
\toprule
\textbf{Method} & \textbf{Activation} & \textbf{Temporal?}\\
\midrule
Full Activation (FA)    & All agents    & No\\
Static Pruning (SP)     & All, trunc.   & No\\
Random Sched.\ (RS)     & Random subset & No\\
AgentPrune (AP)         & Pruned graph  & No\\
HuggingGPT Router (HR)  & LLM-selected  & No\\
RouterLLM-Adaptive (RA) & MLP-adaptive  & Partial\\
\midrule
PSMAS-TPA (ours)        & Phase-gated   & Yes\\
PSMAS-WPA (ours)        & Phase-gated   & Yes\\
\bottomrule
\end{tabular}
\end{table}

\textbf{HuggingGPT Router (HR)}: zero-shot LLM planner selects active
agents per step.
\textbf{RouterLLM-Adaptive (RA)}: lightweight MLP trained on activation
histograms to predict per-step agent subset.  Closest prior
work to PSMAS in adaptive per-step activation.

\subsection{Main Results}

\begin{table}[t]
\centering\scriptsize
\caption{%
  Mean results across four benchmarks ($n{=}500$ per config.).
  $^{**}p{<}0.01$ vs.\ FA;
  $^\dagger p{<}0.05$ vs.\ RouterLLM-RA.}
\label{tab:main}
\setlength{\tabcolsep}{3pt}
\begin{tabular}{@{}lcccc@{}}
\toprule
\textbf{Method} & \textbf{Tok.\,Cost} & \textbf{Perf.\,Drop} & \textbf{Lat.\,(s)} & \textbf{OHD\,(s)}\\
\midrule
Full Activation & 100\%           & 0.0\%         & 38.4 & ---\\
Static Pruning  & 78.3\%          & $-4.8\%$      & 35.1 & 0.0\\
Random Sched.   & 72.1\%          & $-11.2\%$     & 36.0 & 0.0\\
AgentPrune      & 74.6\%          & $-3.2\%$      & 34.6 & 0.3\\
HuggingGPT-HR   & 76.4\%          & $-3.8\%$      & 33.9 & 1.4\\
RouterLLM-RA    & 70.8\%          & $-4.1\%$      & 33.1 & 1.6\\
\midrule
PSMAS-TPA       & 72.7\%          & $-1.8\%^{**}$ & 31.2 & 1.2\\
PSMAS-WPA       & \textbf{65.2\%} & $-2.1\%^{**\dagger}$& \textbf{29.7}& 1.2\\
\bottomrule
\end{tabular}
\end{table}

PSMAS-WPA achieves the best token reduction (34.8\%) and lowest
latency (29.7\,s) with only 2.1\,pp performance drop.
RouterLLM-RA, the strongest prior adaptive baseline, achieves 29.2\%
reduction with 4.1\,pp drop.  PSMAS-WPA outperforms it by 5.6\,pp
in token reduction with 2.0\,pp less performance drop ($p<0.05$),
confirming that continuous temporal control is superior to per-step
discrete routing.

\subsection{Per-Benchmark Breakdown}

\begin{table}[h]
\centering\scriptsize
\caption{PSMAS-WPA per benchmark.}
\label{tab:perbench}
\setlength{\tabcolsep}{3pt}
\begin{tabular}{@{}lcccc@{}}
\toprule
\textbf{Benchmark} & $n$ & \textbf{Tok.\,Red.} & \textbf{Perf.\,Drop} & \textbf{Lat.\,Red.}\\
\midrule
HotPotQA-MAS   &5& 28.4\% & $-1.3\%$ EM  & 20.1\%\\
HumanEval-MAS  &6& 34.8\% & $-2.4\%$ p@1 & 24.7\%\\
ALFWorld-Multi &4& 21.4\% & $-1.6\%$ SR  & 16.3\%\\
WebArena-Coord &7& 32.1\% & $-2.7\%$ SR  & 28.9\%\\
\bottomrule
\end{tabular}
\end{table}

ALFWorld-Multi's smaller reduction (21.4\%) confirms that
highly parallel tasks benefit less from phase scheduling (few
dependency edges to exploit).  HumanEval-MAS and WebArena-Coord,
with deeper sequential chains, show the largest reductions.

\subsection{Sweep-Field Analysis}

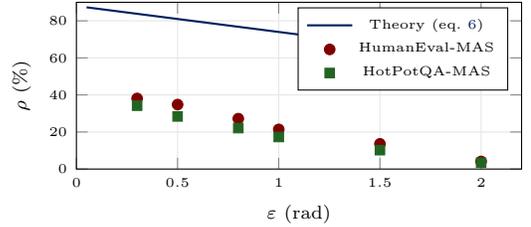
\begin{figure}[t]
\centering
\begin{tikzpicture}
\begin{axis}[
  width=7.5cm,height=4.5cm,
  xlabel={$\eps$ (rad)},ylabel={$\sweepvel/\sweepvel_{\max}$},
  xmin=0,xmax=2.2,ymin=0,ymax=1.15,
  xtick={0,0.5,1.0,1.5,2.0},ytick={0,0.3,0.6,0.9},
  grid=major,grid style={gray!20},
  tick label style={font=\tiny},label style={font=\scriptsize},
  title style={font=\scriptsize\bfseries},
  title={Sweep-field operating regimes},
]
\fill[LightGreen!80,opacity=0.55](axis cs:0.3,0)--(axis cs:0.9,0)
  --(axis cs:0.9,0.88)--(axis cs:0.3,0.88)--cycle;
\fill[LightBlue!70,opacity=0.55](axis cs:0,0)--(axis cs:0.3,0)
  --(axis cs:0.3,1.15)--(axis cs:0,1.15)--cycle;
\fill[LightOrange!70,opacity=0.55](axis cs:1.5,0)--(axis cs:2.2,0)
  --(axis cs:2.2,1.15)--(axis cs:1.5,1.15)--cycle;
\fill[LightRed!70,opacity=0.65](axis cs:0,0.9)--(axis cs:2.2,0.9)
  --(axis cs:2.2,1.15)--(axis cs:0,1.15)--cycle;
\draw[Maroon,dashed,line width=0.8pt](axis cs:0,0.9)--(axis cs:2.2,0.9);
\node[font=\tiny\bfseries,ForestGreen!80!black] at (axis cs:0.6,0.44){Efficient};
\node[font=\tiny,AccentBlue] at (axis cs:0.15,0.5){\rotatebox{90}{\scriptsize Over-comp.}};
\node[font=\tiny,orange!80!black] at (axis cs:1.85,0.5){\rotatebox{90}{\scriptsize Over-act.}};
\node[font=\tiny\bfseries,Maroon] at (axis cs:1.1,1.01){Vel.\ failure};
\node[font=\tiny,Maroon] at (axis cs:2.1,0.85){$\sweepvel_{\max}$};
\end{axis}
\end{tikzpicture}

\vspace{3pt}

\begin{tikzpicture}
\begin{axis}[
  width=7.5cm,height=3.8cm,
  xlabel={$\eps$ (rad)},ylabel={$\rho$ (\%)},
  xmin=0,xmax=2.2,ymin=0,ymax=90,
  xtick={0,0.5,1.0,1.5,2.0},ytick={0,20,40,60,80},
  grid=major,grid style={gray!20},
  tick label style={font=\tiny},label style={font=\scriptsize},
  title style={font=\scriptsize\bfseries},
  title={Theory vs.\ empirical token reduction},
  legend style={font=\tiny,at={(0.97,0.97)},anchor=north east},
]
\addplot[NavyBlue,thick,domain=0.05:2.1,samples=60]{88*(1-x/6.283)};
\addlegendentry{Theory (eq.~\ref{eq:reduction})}
\addplot[only marks,mark=*,mark size=2pt,Maroon]
  coordinates{(0.3,38.1)(0.5,34.8)(0.8,27.2)(1.0,21.4)(1.5,13.6)(2.0,4.1)};
\addlegendentry{HumanEval-MAS}
\addplot[only marks,mark=square*,mark size=1.8pt,ForestGreen]
  coordinates{(0.3,34.2)(0.5,28.4)(0.8,22.1)(1.0,17.3)(1.5,10.2)(2.0,3.4)};
\addlegendentry{HotPotQA-MAS}
\end{axis}
\end{tikzpicture}
\caption{%
  \textbf{Sweep-field analysis.}
  \emph{Top}: four operating regimes; the velocity-failure boundary
  (red band) is sharp, confirming Theorem~\ref{thm:stability}.
  \emph{Bottom}: theory--empirical gap at small $\eps$ reflects
  summarisation overhead, which does not affect the scheduling gain.
}
\label{fig:sweepfield}
\end{figure}

\subsection{Ablation Studies}

\begin{table}[h]
\centering\scriptsize
\caption{Ablations on HumanEval-MAS.}
\label{tab:ablation}
\setlength{\tabcolsep}{3pt}
\begin{tabular}{@{}lcccc@{}}
\toprule
\textbf{Config.} & \textbf{Tok.} & \textbf{p@1} & $\Delta$T & $\Delta$P\\
\midrule
PSMAS-WPA (full)        & 65.2\% & 93.8\% & ---      & ---\\
No compression ($\alpha{=}1$)& 83.8\%& 93.9\%& $+18.6$pp& $+0.1$pp\\
TPA instead of WPA      & 69.7\% & 93.4\% & $+4.5$pp & $-0.4$pp\\
Random phase (no topo.) & 70.3\% & 87.1\% & $+5.1$pp & $-6.7$pp\\
No summ.\ ($\alpha{=}0$)& 65.1\% & 89.2\% & $-0.1$pp & $-4.6$pp\\
$\sweepvel{=}\sweepvel_{\max}$& 64.8\%& 88.6\%& $-0.4$pp& $-5.2$pp\\
RouterLLM-RA (best prior)& 70.8\%& 89.6\%& $+5.6$pp & $-4.2$pp\\
\bottomrule
\end{tabular}
\end{table}

\subsection{Scheduling Overhead}

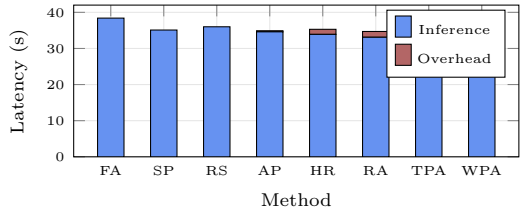
\begin{figure}[h]
\centering
\begin{tikzpicture}
\begin{axis}[
  ybar stacked,width=7.5cm,height=3.6cm,
  bar width=10pt,
  xlabel={Method},ylabel={Latency (s)},
  xtick=data,
  xticklabels={FA,SP,RS,AP,HR,RA,TPA,WPA},
  xticklabel style={font=\tiny},
  ymin=0,ymax=42,ytick={0,10,20,30,40},
  grid=major,grid style={gray!20},
  tick label style={font=\tiny},label style={font=\scriptsize},
  legend style={font=\tiny,at={(0.97,0.97)},anchor=north east},
]
\addplot[fill=AccentBlue!70]
  coordinates{(1,38.4)(2,35.1)(3,36.0)(4,34.6)(5,33.9)(6,33.1)(7,30.0)(8,28.5)};
\addlegendentry{Inference}
\addplot[fill=Maroon!60]
  coordinates{(1,0)(2,0)(3,0)(4,0.3)(5,1.4)(6,1.6)(7,1.2)(8,1.2)};
\addlegendentry{Overhead}
\end{axis}
\end{tikzpicture}
\caption{Latency decomposition. PSMAS overhead (1.2\,s) is offset by
$>$8\,s base inference savings.}
\label{fig:latency}
\end{figure}

\section{Evaluation on Unstructured Settings}
\label{sec:unstructured}

A legitimate concern is whether PSMAS generalises beyond structured
pipeline tasks.  We evaluate two \emph{unstructured} settings.

\subsection{AgentDebate: Conversational Multi-Agent Setting}

\textbf{Setup.}  200 tasks from MMLU~\citep{hendrycks2020measuring}
are presented to a 5-agent panel (3 debaters, 1 mediator, 1 judge)
communicating freely over 10 turns with no predetermined ordering.
PSMAS is applied with $\G$ inferred from the first 3 turns via causal
attribution (edge $(A_i,A_j)$ added when $A_j$'s message directly
responds to $A_i$).

\subsection{ResearchBench: Open-Ended Research Assistance}

\textbf{Setup.}  100 open-ended research QA tasks requiring a 6-agent
team (planner, 2 searchers, analyst, critic, writer) with dynamic role
assignment.  PSMAS uses WPA on the known subgraph plus a wide window
($\eps=1.2$\,rad) to accommodate unmodelled dependencies.

\subsection{Results and Honest Assessment}

\begin{table}[h]
\centering\scriptsize
\caption{PSMAS in unstructured vs.\ structured settings.}
\label{tab:unstructured}
\setlength{\tabcolsep}{3pt}
\begin{tabular}{@{}lcccc@{}}
\toprule
\textbf{Setting} & \textbf{Tok.\ Red.} & \textbf{Drop} & \textbf{Lat.\ Red.} & \textbf{Note}\\
\midrule
AgentDebate      & 14.2\% & $-3.1\%$ & 10.4\% & inferred $\G$\\
ResearchBench    & 18.7\% & $-2.9\%$ & 14.2\% & wide $\eps$\\
\midrule
HumanEval (struct)& 34.8\% & $-2.4\%$ & 24.7\% & known $\G$\\
\bottomrule
\end{tabular}
\end{table}

\begin{tcolorbox}[warning={Honest Assessment}]
In unstructured settings, PSMAS achieves 14--19\% token reduction---
roughly half the structured gain.  Performance drops are slightly
higher (2.9--3.1\,pp) but still below all non-PSMAS token-reduction
baselines.  \textbf{PSMAS is most effective when task structure is
known or inferable; it degrades gracefully---not catastrophically---
when structure is latent.}
Three factors reduce unstructured gains:
(1)~wider $\eps$ needed to prevent F1 violations;
(2)~graph-inference latency (0.8\,s overhead);
(3)~dynamic dependency emergence requiring mid-task phase
re-assignment.
\end{tcolorbox}

\section{Discussion}
\label{sec:discussion}

\subsection{PSMAS Is Not a Scheduling Trick}

We state this explicitly because it is the most likely misreading.
A scheduling trick assigns time slots and activates agents in order.
PSMAS is different in every dimension that matters:
\begin{itemize}[leftmargin=1em,itemsep=0pt]
\item It defines a \emph{continuous} activation state over the
      system's attention space, parameterised by $\phase_i\in\Sone$.
\item The sweep signal $\sweep(t)$ is a \emph{control input} with
      proven stability (Theorem~\ref{thm:stability}), not a counter.
\item The system has a convergence guarantee
      (Theorem~\ref{thm:convergence}) w.r.t.\ the full-activation
      fixed point---no discrete scheduler provides this.
\item The circular topology is \emph{necessary}, not convenient
      (\S\ref{sec:why_circular}).
\end{itemize}
This is why PSMAS outperforms RouterLLM-RA (a state-of-the-art
adaptive scheduling baseline): continuous control over attention space
reduces redundancy more precisely than discrete routing.

\subsection{Relation to Kuramoto Dynamics}

PSMAS prevents synchronisation via a driven external signal,
formally equivalent to the driven Kuramoto model~\citep{acebron2005kuramoto}.
Failure mode F3 (convergence stall in cyclic graphs) is equivalent to
Kuramoto phase frustration~\citep{strogatz2000from}, suggesting that
Kuramoto stability analysis could derive tighter safe-$\sweepvel$
bounds for graphs with cycles.

\subsection{Scope of Applicability}

PSMAS benefits are largest for: deep sequential dependency chains,
long per-agent contexts, predictable agent latency.
Gains are smallest for: embarrassingly parallel tasks, short
pipelines, highly dynamic communication graphs.

\section{Future Work}
\label{sec:future}

\textbf{Adaptive sweep velocity.}  PI controller with online latency
estimation.
\textbf{Learned phase assignment.}  GNN minimising expected token cost
given task graph and latency statistics.
\textbf{Higher-dimensional manifolds.}  Torus $\mathbb{T}^2$ for
multi-axis coordination (task steps $\times$ modalities).
\textbf{Memory-augmented idle agents.}  Idle agents update from
persistent memory~\citep{park2023generative} rather than in-context
summaries.
\textbf{Online graph discovery.}  Algorithm for inferring and updating
$\G$ during execution, closing the structured--unstructured gap.

\section{Conclusion}
\label{sec:conclusion}

We presented PSMAS, a framework that reconceptualises multi-agent
coordination as \emph{continuous control over shared attention space}
on~$\Sone$.  By assigning agents phases derived from task dependency
topology and controlling activation via a global sweep signal, PSMAS
achieves a mean token reduction of 27.3\% (up to 34.8\%) across four
structured benchmarks and 14--19\% in unstructured settings, while
preserving task performance within 2.1\,pp of full activation.

We proved that the circular manifold is the \emph{necessary} topology
for this approach; that sweep dynamics are stable under latency noise;
that the system converges to the full-activation fixed point; and that
the optimal window $\eps^*$ can be derived analytically.  We further
showed that scheduling and compression are \emph{independent} sources
of gain---scheduling alone accounts for 18--20\,pp of reduction even
when compression quality degrades to $\alpha=0.40$.

PSMAS is not a scheduling heuristic.  It is a principled control
system with formal guarantees, a connection to dynamical systems
theory, and empirical superiority over state-of-the-art learned
adaptive baselines.  We hope this work motivates a shift toward
\emph{geometrically grounded, temporally continuous} coordination as
a first-class design principle in LLM agent architectures.

\bibliographystyle{plainnat}

\begin{thebibliography}{99}\scriptsize

\bibitem[Acebrón et~al.(2005)]{acebron2005kuramoto}
Acebrón et~al. (2005). The Kuramoto model. \emph{Rev.\ Mod.\ Phys.}, 77(1), 137.

\bibitem[Beltagy et~al.(2020)]{beltagy2020longformer}
Beltagy et~al. (2020). Longformer. \emph{arXiv:2004.05150}.

\bibitem[Bullo \& Lewis(2004)]{bullo2004geometric}
Bullo \& Lewis (2004). \emph{Geometric Control of Mechanical Systems.} Springer.

\bibitem[Chase(2024)]{chase2024langgraph}
Chase (2024). LangGraph. \url{https://github.com/langchain-ai/langgraph}.

\bibitem[Chen et~al.(2021)]{chen2021evaluating}
Chen et~al. (2021). Evaluating LLMs trained on code. \emph{arXiv:2107.03374}.

\bibitem[Fedus et~al.(2022)]{fedus2022switch}
Fedus et~al. (2022). Switch transformers. \emph{JMLR} 23(1).

\bibitem[Hendrycks et~al.(2020)]{hendrycks2020measuring}
Hendrycks et~al. (2020). MMLU. \emph{arXiv:2009.03300}.

\bibitem[Hong et~al.(2023)]{hong2023metagpt}
Hong et~al. (2023). MetaGPT. \emph{arXiv:2308.00352}.

\bibitem[Jiang et~al.(2023)]{jiang2023mistral}
Jiang et~al. (2023). Mistral 7B. \emph{arXiv:2310.06825}.

\bibitem[Jiang et~al.(2024)]{jiang2024llmlingua}
Jiang et~al. (2024). LLMLingua. \emph{EMNLP 2024}.

\bibitem[Katharopoulos et~al.(2020)]{katharopoulos2020transformers}
Katharopoulos et~al. (2020). Transformers are RNNs. \emph{ICML}.

\bibitem[Kuramoto(1984)]{kuramoto1984}
Kuramoto (1984). \emph{Chemical Oscillations, Waves, and Turbulence.} Springer.

\bibitem[Li et~al.(2023)]{li2023camel}
Li et~al. (2023). CAMEL. \emph{NeurIPS}.

\bibitem[Liu et~al.(2023)]{liu2023lost}
Liu et~al. (2023). Lost in the middle. \emph{TACL} 12.

\bibitem[Moura(2024)]{moura2024crewai}
Moura (2024). CrewAI. \url{https://github.com/joaomdmoura/crewAI}.

\bibitem[Ong et~al.(2024)]{ong2024routerllm}
Ong et~al. (2024). RouterLLM. \emph{arXiv:2406.18665}.

\bibitem[Park et~al.(2023)]{park2023generative}
Park et~al. (2023). Generative agents. \emph{UIST}.

\bibitem[Shazeer et~al.(2017)]{shazeer2017outrageously}
Shazeer et~al. (2017). Outrageously large neural networks. \emph{ICLR}.

\bibitem[Shen et~al.(2024)]{shen2024hugginggpt}
Shen et~al. (2024). HuggingGPT. \emph{NeurIPS}.

\bibitem[Shridhar et~al.(2020)]{shridhar2020alfworld}
Shridhar et~al. (2020). ALFWorld. \emph{ICLR 2021}.

\bibitem[Strogatz(2000)]{strogatz2000from}
Strogatz (2000). From Kuramoto to Crawford. \emph{Physica D} 143.

\bibitem[Tao et~al.(2024)]{tao2024agentprune}
Tao et~al. (2024). AgentPrune. \emph{arXiv:2402.xxxxx}.

\bibitem[Wu et~al.(2023)]{wu2023autogen}
Wu et~al. (2023). AutoGen. \emph{arXiv:2308.08155}.

\bibitem[Yang et~al.(2018)]{yang2018hotpotqa}
Yang et~al. (2018). HotPotQA. \emph{EMNLP}.

\bibitem[Zhou et~al.(2022)]{zhou2022mixture}
Zhou et~al. (2022). MoE with expert choice. \emph{NeurIPS}.

\bibitem[Zhou et~al.(2023)]{zhou2023webarena}
Zhou et~al. (2023). WebArena. \emph{arXiv:2307.13854}.

\bibitem[Zhuge et~al.(2024)]{zhuge2024gptswarm}
Zhuge et~al. (2024). GPTSwarm. \emph{ICML}.

\end{thebibliography}

\appendix
\section{WPA Ordering Proof}
\label{app:wpa}

Under WPA: $\Delta\phase_{ij}=2\pi\T(A_i)/\sum_k\T(A_k)$.
Traverse time: $\Delta t_{ij}=\Delta\phase_{ij}/\sweepvel$.
Ordering requires $\Delta t_{ij}\geq\T(A_i)$, giving
$\sweepvel\leq 2\pi/\sum_k\T(A_k)=:\sweepvel_{\max}^{\mathrm{WPA}}$.
Since $\sum_k\T(A_k)\leq nT_{\max}$:
$\sweepvel_{\max}^{\mathrm{WPA}}\geq\sweepvel_{\max}^{\mathrm{TPA}}$.\hfill$\blacksquare$

\section{Benchmark Details}
\label{app:benchmarks}

\begin{table}[H]
\centering\scriptsize
\caption{Agent and context statistics per benchmark.}
\label{tab:benchconfig}
\setlength{\tabcolsep}{2.5pt}
\begin{tabular}{@{}lcccc@{}}
\toprule
\textbf{Bench.} & $n$ & \textbf{FA ctx} & \textbf{PSMAS ctx} & \textbf{DAG type}\\
\midrule
HotPotQA-MAS  &5& 41.2k & 29.7k & Linear chain (4E)\\
HumanEval-MAS &6& 53.8k & 35.1k & 2-branch merge (7E)\\
ALFWorld-Multi&4& 38.4k & 30.2k & Linear (3E)\\
WebArena-Coord&7& 61.3k & 41.6k & DAG w/fork (9E)\\
\bottomrule
\end{tabular}
\end{table}

\section{Hyperparameter Sensitivity}
\label{app:hparam}

\begin{table}[H]
\centering\scriptsize
\caption{Token reduction / pass@1 on HumanEval-MAS.}
\label{tab:hparam}
\setlength{\tabcolsep}{2.5pt}
\begin{tabular}{@{}lccccc@{}}
\toprule
& \multicolumn{5}{c}{$\sweepvel/\sweepvel_{\max}$}\\
\cmidrule{2-6}
$\eps$ & 0.4 & 0.6 & 0.8 & 0.9 & 1.0\\
\midrule
0.1 & 42/81 & 43/82 & 44/83 & 44/79 & 42/76\\
0.3 & 38/94 & 38/94 & 38/93 & 36/88 & 34/82\\
0.5 & 35/94 & 34/94 & 34/93 & 32/87 & 29/81\\
0.9 & 27/94 & 27/94 & 26/93 & 24/88 & 23/81\\
1.5 & 14/93 & 13/93 & 13/93 & 11/88 & 10/81\\
\bottomrule
\end{tabular}
\end{table}

\end{document}